\documentclass[letterpaper, 10 pt, conference]{ieeeconf}

\IEEEoverridecommandlockouts
\usepackage{cite}
\usepackage{float}
\usepackage{graphicx}
\usepackage{amssymb}

\usepackage{amsmath}
\usepackage{mathtools}  
\usepackage{tikz}

\usepackage{nccmath}

\usepackage{amsthm}
\usepackage{tabularray}
\usepackage{booktabs}
\usepackage{makecell}

\usepackage[usenames,dvipsnames,svgnames,table]{xcolor}
\theoremstyle{plain}

\newtheorem{proposition}{Proposition}

\usepackage{float}
\usetikzlibrary{arrows.meta,positioning,calc}
\floatstyle{plaintop}
\usepackage{MnSymbol}%
\usepackage{wasysym}%
\usepackage{caption}
\usepackage{subcaption}
\usepackage{graphicx}
\usepackage{import}
\usepackage{hyperref}

\usepackage{algorithm}
\usepackage{algpseudocode}

\definecolor{deepblue}{rgb}{0,0,1}
\usepackage{dblfloatfix}

\definecolor{SpringGreen}{RGB}{40,160,80} 
\definecolor{SkyBlue}{RGB}{60,130,200}   
\definecolor{Khaki}{RGB}{160,100,70}     
\definecolor{Lavender}{RGB}{180,100,180}  
\definecolor{FruitOrange}{RGB}{245,218,95}

\begin{document}

\title{From Ellipsoids to Midair Control of Dynamic Hitches}

\author{Jiawei Xu, Subhrajit Bhattacharya, and David Salda\~{n}a
    \thanks{
        J. Xu, S. Bhattacharya, and D. Salda\~{n}a are with the Autonomous and Intelligent Robotics Laboratory (AIRLab), Lehigh University, PA, 18015, USA. E-mail: \texttt{\{jix519, sub216, saldana\}@lehigh.edu}
    }
\thanks{* The authors gratefully acknowledge the support of the NSF Awards~2322840 and 2442475.}
}

\maketitle

\begin{abstract}
The ability to manipulate and interlace cables using aerial vehicles can greatly improve aerial transportation tasks. Such interlacing cables create hitches by winding two or more cables around each other, which can enclose payloads or can further develop into knots. Dynamic modeling and control of such hitches are key to mastering inter-cable interactions in the context of cable-suspended aerial manipulation. This paper introduces an ellipsoid-based kinematic model to connect the geometric nature of a hitch created by two cables and the dynamics of the hitch driven by four aerial vehicles, which reveals the control-affine form of the system. As the constraint for maintaining tension of a cable is also control-affine, we design a quadratic programming-based controller that combines Control Lyapunov and High-Order Control Barrier Functions (CLF-HOCBF-QP) to precisely track a desired hitch position and system shape while enforcing safety constraints like cable tautness. We convert desired geometric reference configurations into target robot positions and introduce a composite error into the Lyapunov function to ensure a relative degree of one to the input.  Numerical simulations validate our approach, demonstrating stable, high-speed tracking of dynamic references.
\end{abstract}

\section{Introduction}{
    Ropes and cables are one of the most widely used soft tools in human history~\cite{fronzaglia2006history}. From catching and securing objects~\cite{dana1922use} to transportation~\cite{7171008}, such long and thin structures as cables are useful in many applications because of their high flexibility and constrained extensibility. This simple yet powerful duality has enabled countless applications, from the capture of objects~\cite{dana1922use}, transportation~\cite{7171008}, to the construction and maintenance of complex structures~\cite{d2024model}. In recent robotics literature, particularly in aerial systems, cable-suspended aerial systems are emerging as a compelling alternative to traditional rigid-link manipulators, thereby expanding the versatility of autonomous agents in various scenarios~\cite{sarkisov2019development,10758214}.
    Moreover, tying cables into knots allows the ad-hoc creation of fixtures and enclosures on objects~\cite{turner1996history}, which facilitates the interaction tasks with distinct objects that lack common grasping points, such as glass bottles, comparable to the capabilities of humans' dexterous hands. 
    
    Advances in the past decade highlight the automation of tying static or quasi-static hitches and knots~\cite{7354218,d2023forming}. A common goal of these approaches is to use aerial vehicles to drive the tethered cables to achieve desired topological configurations~\cite{9981363}. In the process, the interlaced cables enforce kinematic constraints to the aerial systems. However, these methods largely rely on quasi-static assumptions~\cite{1638335}. As studied in many research on cable-driven load manipulation, considering the full dynamics of the system not only extends the agility of the approaches, but also enables stability guarantee~\cite{drones8020035}. The quasi-static simplification of current study in inter-cable manipulation precludes high-speed operation especially when the kinematic model is intrinsically coupled with the system dynamics~\cite{sun2025agile}, despite the agility of aerial systems.
    
    \begin{figure}[t]
        \centering
        \includegraphics[width=\linewidth]{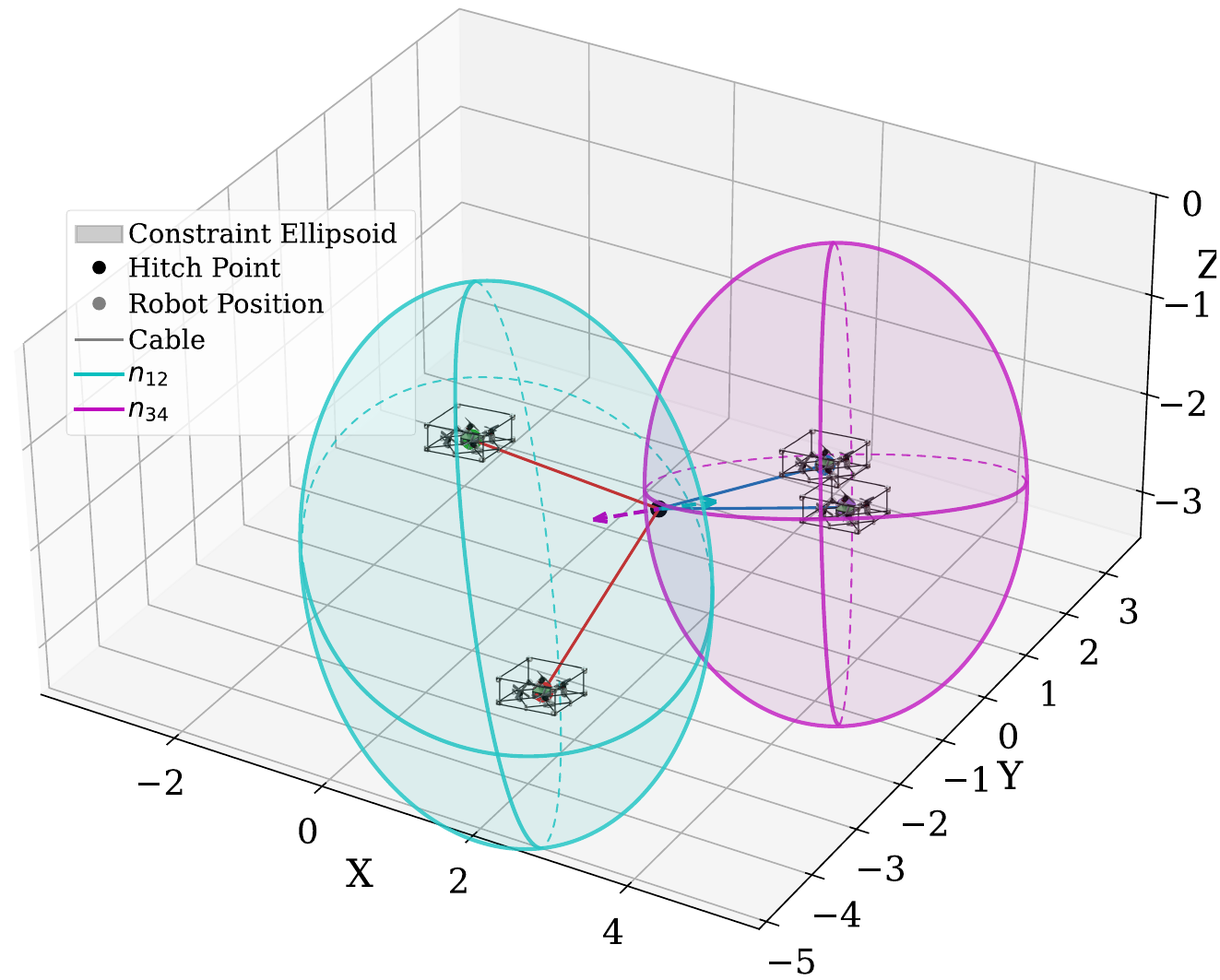}
        \caption{Illustration of the cable-suspended system with the constraint ellipsoids overlay. Line segments of the same color denote a single cable connecting two robots. The cables interlace/wind around each other at the \emph{hitch point}. The two dashed arrows represent the normal vectors of the two ellipsoids at the hitch point.}
        \label{fig:traj}
    \end{figure}
    In this paper, we shift the paradigm from the procedural steps of creating a geometric cable enclosure to the continuous control of the dynamic hitch as the entire cable-suspended system moves through space. We define a \textbf{\emph{hitch point}} or simply, \emph{a hitch}, to be the point at which two cables interlace/wind around each other (Fig. ~\ref{fig:traj}). As a hitch is the foundational element of knots and more complex inter-cable structures, understanding and controlling its  dynamics is a solid milestone toward aerial cable manipulation. Therefore, we focus on the manipulation of a hitch, as well as the cable configuration that satisfies the kinematic constraints.

    The contribution of this paper is twofold. First, from the pin-and-string ellipsoid kinematics model, by introducing a ``virtual mass'' at the hitch to emulate the resistance of motion, we derive the full dynamics of the system of four robots tethered to two interlacing cables, using Newton's formulation. Second, recognizing the control-affine form of the system, we develop a Control Lyapunov Function-High-Order Control Barrier Function-based Quadratic Programming (CLF-HOCBF-QP) controller to drive the system into reference configurations that specify the hitch position and the ellipsoids' shape. We derive the desired robot position from a reference configuration and control its input forces by solving the QP online. To ensure the Lyapunov function has a relative degree one to the input, we integrate a cascade error in the Lyapunov function, and show the convergence of the robot to the desired positions, given the convergence of the Lyapunov function constrained by the QP. 
    
    We implement the system dynamics and the CLF-HOCBF-QP controller in a numerical environment, and evaluate the Lyapunov function value and the reference configuration error under different scenarios, including tracking a static reference, tracking a slow reference while receiving a high Gaussian noise at the hitch, and tracking dynamic references with increasing speeds.
}

\section{Hitch Dynamics}{
    \begin{figure}[t]
        \centering
        \vspace{0.5em}
        \includegraphics[width=\linewidth]{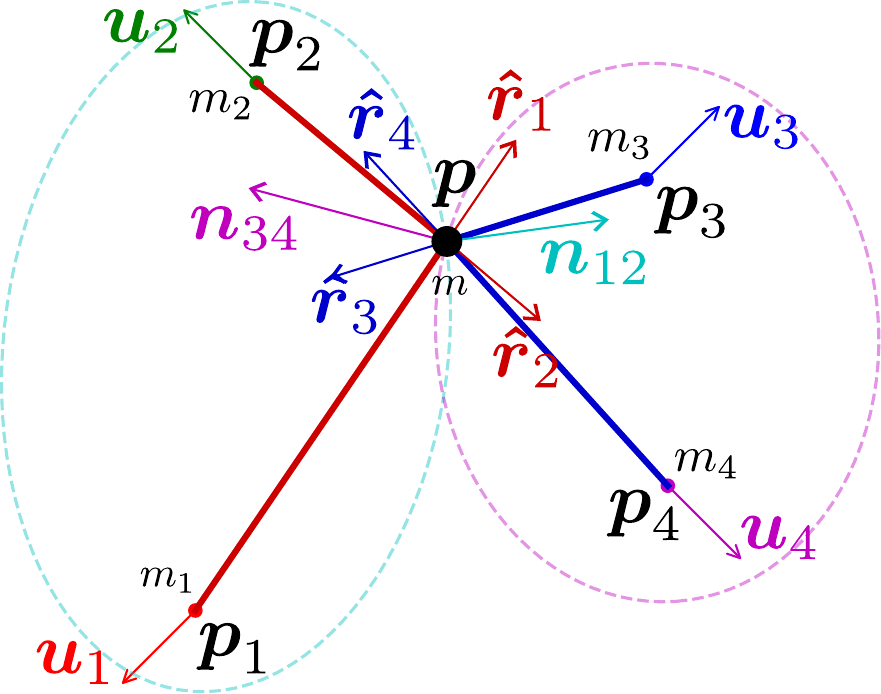}
        \caption{Two cables interlacing at the hitch $\boldsymbol{p}$, each connecting two robots. Each robots apply an input force vector $\boldsymbol{u}_i$, illustrated in 2D.}
        \label{fig:cables}
    \end{figure}
    In our setting, four quadrotors are connected to the ends of two cables. These cables interlace, creating a hitch at their intersection. The hitch’s configuration is constrained by ellipsoidal geometry when the cables are in tension (see Fig.~\ref{fig:traj} for illustration).
    For each pair of robots and its connecting taut cable, the summation of distance from any point on the cable to the two robots is equal to the constant total cable length. The set of all possible positions of such a point turns out to be an ellipse in 2D, or a prolate spheroid in 3D~\cite{freksa2019geometric,hilbert2021geometry} which is one type of ellipsoids. For conciseness, we refer the sets of both kinds an ``ellipsoid'' for the rest of this paper. As we consider a hitch formed by two winding cables on both cables geometrically, the position of the hitch must belong to both ellipsoid sets. In other words, the hitch is at the intersection of the two ellipsoids.

    Formally, since a hitch is formed by two interlacing taut cables driven by two robots each, its position $\boldsymbol{p}\in\mathbb{R}^n$ is determined by the robot positions and the cable lengths, where the dimension $n = 2$ represents a planar configuration and $n = 3$ represents a more general $3D$ configuration. With inextensible cables, $\boldsymbol{p}$ must satisfy 
    \begin{eqnarray}
        \begin{aligned}
            \|\boldsymbol{p - p}_1\| + \|\boldsymbol{p - p}_2\| - l_{12}& = 0,\\
            \|\boldsymbol{p - p}_3\| + \|\boldsymbol{p - p}_4\| - l_{34}& = 0,
        \end{aligned}
        \label{eq:pinandstring}
    \end{eqnarray}
    where robots $1, 2$ are connected to cable $c_{12}$ at $\boldsymbol{p}_1$, $\boldsymbol{p}_2$, respectively; robots $3, 4$ are connected to cable $c_{34}$ at $\boldsymbol{p}_3$, $\boldsymbol{p}_4$, respectively. We refer to~\eqref{eq:pinandstring} as the kinematics of the hitch. Taking double derivatives of~\eqref{eq:pinandstring}, we can rearrange the equations with respect to $\boldsymbol{\ddot{p}}$ to reveal the affine form,
    \begin{equation}
        \boldsymbol{N\ddot{p} = b}, \text{ where }
        \label{eq:affine}
    \end{equation}
    \begin{eqnarray}
        {\begin{aligned}
            \boldsymbol{N} &= \begin{bmatrix}
                \boldsymbol{\hat{r}}_1^\top + \boldsymbol{\hat{r}}_2^\top\\
                \boldsymbol{\hat{r}}_3^\top + \boldsymbol{\hat{r}}_4^\top
            \end{bmatrix}\in\mathbb{R}^{2\times n},\\
            \boldsymbol{b} &= \begin{bmatrix}
                \boldsymbol{\hat{r}}_1^\top\boldsymbol{\boldsymbol{\ddot{p}}}_1 + \boldsymbol{\hat{r}}_2^\top\boldsymbol{\boldsymbol{\ddot{p}}}_2 - \boldsymbol{\dot{\hat{r}}}_1^\top\boldsymbol{\dot{r}}_1 - \boldsymbol{\dot{\hat{r}}}_2^\top\boldsymbol{\dot{r}}_2\\
                \boldsymbol{\hat{r}}_3^\top\boldsymbol{\boldsymbol{\ddot{p}}}_3 + \boldsymbol{\hat{r}}_4^\top\boldsymbol{\boldsymbol{\ddot{p}}}_4 - \boldsymbol{\dot{\hat{r}}}_3^\top\boldsymbol{\dot{r}}_3 - \boldsymbol{\dot{\hat{r}}}_4^\top\boldsymbol{\dot{r}}_4
            \end{bmatrix}\in\mathbb{R}^{2}, \text{ and }
        \end{aligned}}
        \label{eq:affine-components}
    \end{eqnarray} 
    \begin{equation}
        \boldsymbol{r}_i = \boldsymbol{p - p}_i;\quad
        \boldsymbol{\hat{r}}_i = \frac{\boldsymbol{r}_i}{\|\boldsymbol{r}_i\|};\quad
        \boldsymbol{\dot{\hat{r}}}_i^\top = \frac{1}{\|\boldsymbol{r}_i\|} \boldsymbol{\dot{r}}_i^\top\left(\boldsymbol{I} - \boldsymbol{\hat{r}}_i\boldsymbol{\hat{r}}^\top_i\right).
    \end{equation}
    Note that $\boldsymbol{\hat r}_1 + \boldsymbol{\hat r}_2$ and $\boldsymbol{\hat r}_3 + \boldsymbol{\hat r}_4$ are the normal vectors of the ellipsoids at $\boldsymbol{p}$, as the gradient of the LHS of~\eqref{eq:pinandstring} at $\boldsymbol{p}$ are
    \begin{eqnarray}
        \begin{aligned}
            \nabla_{\boldsymbol{p}}\left(\|\boldsymbol{p - p}_1\| + \|\boldsymbol{p - p}_2\| - l_{12}\right) &= \boldsymbol{\hat{r}}_1 + \boldsymbol{\hat{r}}_2\\
            \nabla_{\boldsymbol{p}}\left(\|\boldsymbol{p - p}_3\| + \|\boldsymbol{p - p}_4\| - l_{34}\right) &= \boldsymbol{\hat{r}}_3 + \boldsymbol{\hat{r}}_4.
        \end{aligned}
    \end{eqnarray} 
    We denote $\boldsymbol{n}_{12}\triangleq\boldsymbol{\hat r}_1 + \boldsymbol{\hat r}_2$, $\boldsymbol{n}_{34}\triangleq\boldsymbol{\hat r}_3 + \boldsymbol{\hat r}_4$, and $\boldsymbol{N} = \left[\boldsymbol{n}_{12}, \boldsymbol{n}_{34}\right]^\top$. Practically, the existence of the hitch requires each of the four cable segments to maintain a length less than the corresponding cable length, i.e., $0 < \|\boldsymbol{r}_1\|, \|\boldsymbol{r}_2\| < l_{12}$ and $0 < \|\boldsymbol{r}_3\|, \|\boldsymbol{r}_4\| < l_{34}$. Fig.~\ref{fig:cables} illustrates these quantities. 

    \subsection{Dynamics}
    \noindent
    The state vector of the cable-suspended system,
    \begin{equation}
        \boldsymbol{x} = \left[\boldsymbol{p}^\top, \boldsymbol{p}_1^\top, \boldsymbol{p}_2^\top, \boldsymbol{p}_3^\top, \boldsymbol{p}_4^\top, \boldsymbol{\dot{p}}^\top, \boldsymbol{\dot{p}}_1^\top, \boldsymbol{\dot{p}}_2^\top, \boldsymbol{\dot{p}}_3^\top, \boldsymbol{\dot{p}}_4^\top\right]^\top\in\mathbb{R}^{10n},\label{eq:state-vector}
    \end{equation}
    is composed of robot and hitch positions,
    its derivative $\boldsymbol{\dot x} = \left[\boldsymbol{\dot{p}}^\top, \boldsymbol{\dot{p}}_1^\top,\dots, \boldsymbol{\dot{p}}_4^\top, \boldsymbol{\ddot{p}}^\top, \boldsymbol{\ddot{p}}_1^\top,\dots, \boldsymbol{\ddot{p}}_4^\top\right]^\top$ involves acceleration of the hitch $\boldsymbol{\ddot{p}}$, and of the robots $\boldsymbol{\ddot{p}}_i$ for $i = 1, \dots, 4$. We denote the tension force along cable segment $(\boldsymbol{p}, \boldsymbol{p}_i)$ as $t_i$, and the external force vector applied on $r_i$ as $\boldsymbol{u}_i\in\mathbb{R}^n$. Assuming~$r_i$ has mass $m_i$, and a virtual mass $m$ is introduced at the hitch to emulate the resistance of the hitch motion from the actuation applied at the cable, we use Newton's equation to describe their motion,
    \begin{eqnarray}
        m_i\boldsymbol{\ddot{p}}_i &=& \boldsymbol{u}_i + t_i\boldsymbol{\hat{r}}_i, \text{ for } i = 1, \dots, 4,\label{eq:mass-acceleration-robots}\\
        m\boldsymbol{\ddot{p}} &=& -\boldsymbol{\hat{R}t} + \boldsymbol{f}_d + \boldsymbol{f}_{ext},\label{eq:mass-acceleration-hitch}
    \end{eqnarray}
    where $\boldsymbol{\hat R} = \left[\boldsymbol{\hat r}_1, \boldsymbol{\hat r}_2, \boldsymbol{\hat r}_3, \boldsymbol{\hat r}_4\right]\in\mathbb{R}^{n\times4}, \boldsymbol{t} = \left[t_1, t_2, t_3, t_4\right]^\top$, $\boldsymbol{f}_d(\boldsymbol{x})$ describes the unknown damping of the hitch motion due to friction between cables, which depends on the state of the cable suspended system, and $\boldsymbol{f}_{ext}$ is the external force applied on the hitch, e.g., through a ring the cables are carrying a payload. 
    In addition, since the tension along a taut cable is constant, we have
    \begin{equation}
        t_1 - t_2 = 0 \text{ and } t_3 - t_4 = 0.\label{eq:constant-tension}
    \end{equation} 
    Note that we can decouple the acceleration of the hitch and the robots from the double derivative of the hitch kinematics~\eqref{eq:affine} by substituting $\boldsymbol{\ddot{p}}$ and $\boldsymbol{\ddot{p}}_i$ in~\eqref{eq:affine} and~\eqref{eq:affine-components} with~\eqref{eq:mass-acceleration-robots} and~\eqref{eq:mass-acceleration-hitch}, which leads to the isolated affine expression of the tension forces with respect to the system state $\boldsymbol{x}$ and input~$\boldsymbol{u}_i$,
    \begin{eqnarray}
        &\boldsymbol{Mt} = \boldsymbol{Cu + w}, \text{ where }\label{eq:tension}\\
        &{\begin{aligned}
            \boldsymbol{M} & = \begin{bmatrix}
                \|\boldsymbol{n}_{12}\|^2 + \frac{m}{m_1} + \frac{m}{m_2}& 
                0 & 
                \boldsymbol{n}_{12}^\top\boldsymbol{n}_{34} & 
                0\\
                \boldsymbol{n}_{34}^\top\boldsymbol{n}_{12} & 
                0 & 
                \|\boldsymbol{n}_{34}\|^2 + \frac{m}{m_3} + \frac{m}{m_4} & 
                0 \\
                1 & -1 & 0 & 0\\
                0 & 0 & 1 & -1
            \end{bmatrix} \\
            \boldsymbol{C} & = -m\begin{bmatrix}
                \begin{bmatrix}
                    \frac{\boldsymbol{\hat{r}}_1^\top}{m_1} & \frac{\boldsymbol{\hat{r}}_2^\top}{m_2}
                    & \boldsymbol{0}_{1\times n} & \boldsymbol{0}_{1\times n}\\
                     
                    \boldsymbol{0}_{1\times n} & \boldsymbol{0}_{1\times n} &
                        \frac{\boldsymbol{\hat{r}}_3^\top}{m_3} & \frac{\boldsymbol{\hat{r}}_4^\top}{m_4}
                \end{bmatrix}\\
                \boldsymbol{0}_{2\times 4n}
            \end{bmatrix}\in\mathbb{R}^{4\times4n}\\
            \boldsymbol{w} & = \begin{bmatrix}
                m\begin{bmatrix}
                    \boldsymbol{\dot{\hat{r}}}_1^\top\boldsymbol{\dot{r}}_1 + \boldsymbol{\dot{\hat{r}}}_2^\top\boldsymbol{\dot{r}}_2\\
                    \boldsymbol{\dot{\hat{r}}}_3^\top\boldsymbol{\dot{r}}_3 + \boldsymbol{\dot{\hat{r}}}_4^\top\boldsymbol{\dot{r}}_4
                \end{bmatrix} + \boldsymbol{N}(\boldsymbol{f}_{d} + \boldsymbol{f}_{ext})\\
                \boldsymbol{0}_{2\times1}
            \end{bmatrix}
        \end{aligned}}
        \label{eq:tension-components}
    \end{eqnarray}
    As $\boldsymbol{M}\in\mathbb{R}^{4\times4}$ is invertible, we can obtain the tension 
    \begin{equation}
        \boldsymbol{t} = \boldsymbol{M}^{-1}\left(\boldsymbol{Cu + w}\right).\label{eq:tension-solution}
    \end{equation}    
    It is crucial to ensure the tension across the two cables is always positive so that our model accurately describes the hitch motion as well as the robot motion, i.e., the cables are taut and $\boldsymbol{t}\succ\boldsymbol{0}$, where ``$\succ$'' represents element-wise greater-than. We remark that by substituting the tension forces $t_i$ obtained from~\eqref{eq:tension-solution} into the Newton's equations~\eqref{eq:mass-acceleration-robots}~\eqref{eq:mass-acceleration-hitch}, the hitch and robot accelerations $\boldsymbol{\ddot{p}}, \boldsymbol{\ddot{p}}_i$ turn out linear to the input $\boldsymbol{u} = \left[\boldsymbol{u}_1^\top, \boldsymbol{u}_2^\top, \boldsymbol{u}_3^\top, \boldsymbol{u}_4^\top\right]^\top$. Therefore, we can describe the system dynamics in the affine form of
    \begin{equation}
        \boldsymbol{\dot{x}} = \boldsymbol{h(x)} + \boldsymbol{B(x)u},
        \label{eq:2-cable-4-robot-system}
    \end{equation}
    where the state drift
    \begin{equation}
        {\boldsymbol{h(x)} = \begin{bmatrix}
            \boldsymbol{\dot{p}}\\
            \boldsymbol{\dot{p}}_1\\
            \vdots\\
            \boldsymbol{\dot{p}}_4\\
            -\frac{1}{m}\boldsymbol{\hat R}\boldsymbol{M}^{-1}\boldsymbol{w} + \frac{\boldsymbol{f}_d}{m} + \frac{\boldsymbol{f}_{ext}}{m}\\
            \boldsymbol{JM}^{-1}\boldsymbol{w}
        \end{bmatrix}}
        \label{eq:drift}
    \end{equation}
    and the input matrix
    \begin{equation}
        {\boldsymbol{B(x)} = \begin{bmatrix}
            \boldsymbol{0}_{5n\times4n}\\
            -\frac{1}{m}\boldsymbol{\hat R}\boldsymbol{M}^{-1}\boldsymbol{C}\\
            \boldsymbol{J}_m + \boldsymbol{JM}^{-1}\boldsymbol{C}
        \end{bmatrix}}
        \label{eq:input-matrix}
    \end{equation} 
    where $\boldsymbol{J} = \text{blkdiag}\left(\boldsymbol{\hat r}_1, \boldsymbol{\hat r}_2, \boldsymbol{\hat r}_3, \boldsymbol{\hat r}_4\right)\in\mathbb{R}^{4n\times4}$ 
    and 
    $\boldsymbol{J}_m = \text{blkdiag}\left(\frac{1}{m_1}\boldsymbol{I}_n, \frac{1}{m_2}\boldsymbol{I}_n, \frac{1}{m_3}\boldsymbol{I}_n, \frac{1}{m_4}\boldsymbol{I}_n\right)\in\mathbb{R}^{4n\times4n}$, both block diagonal matrices. We denote $\boldsymbol{B_p} = -\frac{1}{m}\boldsymbol{\hat R}\boldsymbol{M}^{-1}\boldsymbol{C}$ the first nonzero block of $\boldsymbol{B}$ as it is associated to the hitch acceleration, and $\boldsymbol{B}_{nonzero}$ the entire nonzero half of $\boldsymbol{B}$. Notably, the affine hitch dynamics in~\eqref{eq:affine} facilitates the straight-forward deployment of optimal control strategies.

}

\section{Control}{
    We focus on controlling the hitch motion. 
    Since the ellipsoid abstraction of the two cables provides an intuition of the cable-suspended system, we design the controller to drive the system to achieve a hitch position $\boldsymbol{p}^{r}$, ellipsoid normal vectors $\boldsymbol{n}_{12}^{r}, \boldsymbol{n}_{34}^{r}$, and major axis lengths $d_{12}^{r}, d_{34}^{r}$. We denote $\left(\boldsymbol{p}^{r}, \boldsymbol{n}_{12}^{r}, \boldsymbol{n}_{34}^{r}, d_{12}^{r}, d_{34}^{r}\right)$ as the \emph{reference} configuration.
    
    \subsection{Equilibrium and nominal input}{
        \label{sec:nominal}
        In the case where no external force is applied on the hitch, i.e., $\boldsymbol{f}_{ext} = \boldsymbol{0}$, we find that the system defined in~\eqref{eq:2-cable-4-robot-system} has a natural equilibrium subspace associated with a static hitch, of which the criteria are obtained from~\eqref{eq:mass-acceleration-hitch}. Letting $\boldsymbol{\ddot{p} = 0}$ and $\boldsymbol{\dot{p} = 0}$, the hitch dynamics reduces to
        $-\sum_{i=1}^4 t_i\boldsymbol{\hat{r}}_i = \boldsymbol{0}.$
        In combination with the constant tension along each cable, i.e., $t_1 = t_2, t_3 = t_4$, we obtain
        \begin{equation}
            t_{1}\boldsymbol{n}_{12} = -t_{3}\boldsymbol{n}_{34}. \label{eq:equilibrium}
        \end{equation}
        Since the tension forces $\boldsymbol{t}\succ\boldsymbol{0}$ are larger than zero, the equilibrium criteria is equivalent to $\boldsymbol{n}_{12}\parallel\boldsymbol{n}_{34}$. Thus, the criteria of $\boldsymbol{p}$ at equilibrium are equivalent to the two ellipsoids cotangent at $\boldsymbol{p}$.
        Similarly, we obtain the nominal input subspace that is associated with zero robot acceleration by setting $\boldsymbol{\ddot p}_i = \boldsymbol{0}$ of~\eqref{eq:mass-acceleration-robots} for $i=1, \dots, 4$, which reduces to $\boldsymbol{u}_i = -t_i\boldsymbol{\hat r}_i$, i.e., the input force of each robot is in the opposite direction of the cable tension.
    }

    \subsection{Feasibility}{
        \label{sec:feasibility}
        Since the tension force along each cable is strictly positive when the cables are taut, the feasible acceleration of the hitch is constrained depending on the cable directions. From the hitch dynamics~\eqref{eq:mass-acceleration-hitch}, we obtain 
        \begin{equation}
            \boldsymbol{\hat{R}t} = -m\boldsymbol{\ddot{p}} + \boldsymbol{f}_d + \boldsymbol{f}_{ext}.\label{eq:rearrange-mass-acceleration-hitch}
        \end{equation}
        To maintain the positiveness of the tension, the achievable hitch acceleration should render the RHS of~\eqref{eq:rearrange-mass-acceleration-hitch} inside the cone enclosed by rays $\boldsymbol{\hat{r}}_1, \boldsymbol{\hat{r}}_2, \boldsymbol{\hat{r}}_3, \boldsymbol{\hat{r}}_4$. Moreover, since $t_1 = t_2, t_3 = t_4$, the achievable hitch acceleration is such that the RHS of~\eqref{eq:rearrange-mass-acceleration-hitch} lies in the planar cone spanned by $\boldsymbol{n}_{12}$ and $\boldsymbol{n}_{34}$. In the spatial scenario, $n = 3$, it means the feasible acceleration of the hitch is limited in a plane. Moreover, when the normal vectors of the two ellipsoids are colinear, i.e., $\boldsymbol{n}_{12}\times\boldsymbol{n}_{34} = \boldsymbol{0}$, the cables' tension cannot contribute to the force in the direction tangential to both ellipsoids, which reduces further the dimension of the feasible hitch acceleration. Fig.~\ref{fig:singular-Bp} shows the plot of the  second smallest singular value of $\boldsymbol{B_p}$ as $\|\boldsymbol{n}_{12}\times\boldsymbol{n}_{34}\|$ changes in 3D while we maintain both normal vectors a magnitude of $\sqrt{2}$. Note that the smallest singular value of $\boldsymbol{B_p}$ is always $0$ in 3D as two vectors spans at most a plane.

        Although the hitch motion is directly affected by the cable tension, it is fundamentally affected by the robot position, constrained by~\eqref{eq:pinandstring}. We study the smallest singular values of the entire nonzero block of the input matrix $\boldsymbol{B}_{nonzero}$ in~\ref{fig:singular-Bnonzero} and find that it is always larger than $1$ regardless of~$n$, meaning that we can always drive the system so that the robot and the hitch achieves desired acceleration together, thus overcoming the infeasible hitch motion problem.
        \begin{figure}
            \vspace{0.5em}
            \centering
            \subfloat[Second smallest singular values of $\boldsymbol{B_p}$\label{fig:singular-Bp}]{
                \centering
                \includegraphics[width=0.45\linewidth]{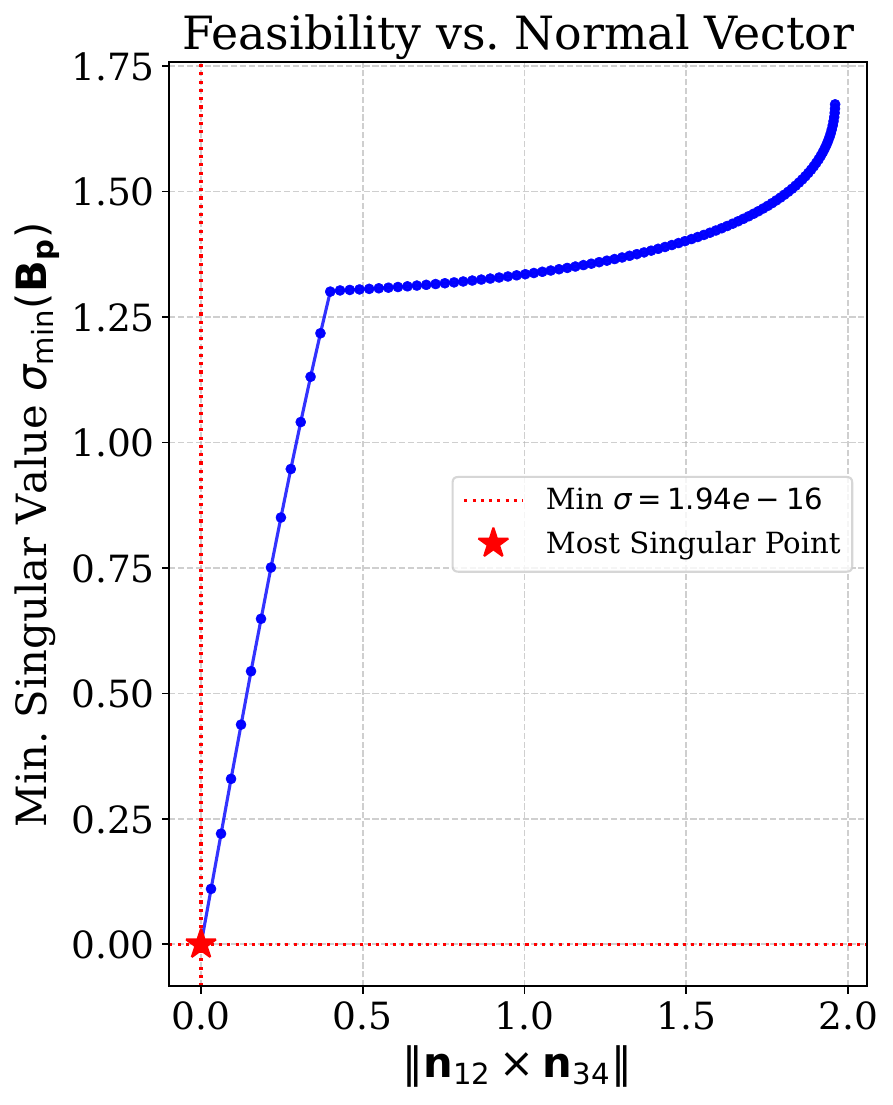}
            }
            \quad
            \subfloat[Smallest singular values of $\boldsymbol{B}_{nonzero}$ \label{fig:singular-Bnonzero}]{
                \centering
                \includegraphics[width=0.45\linewidth]{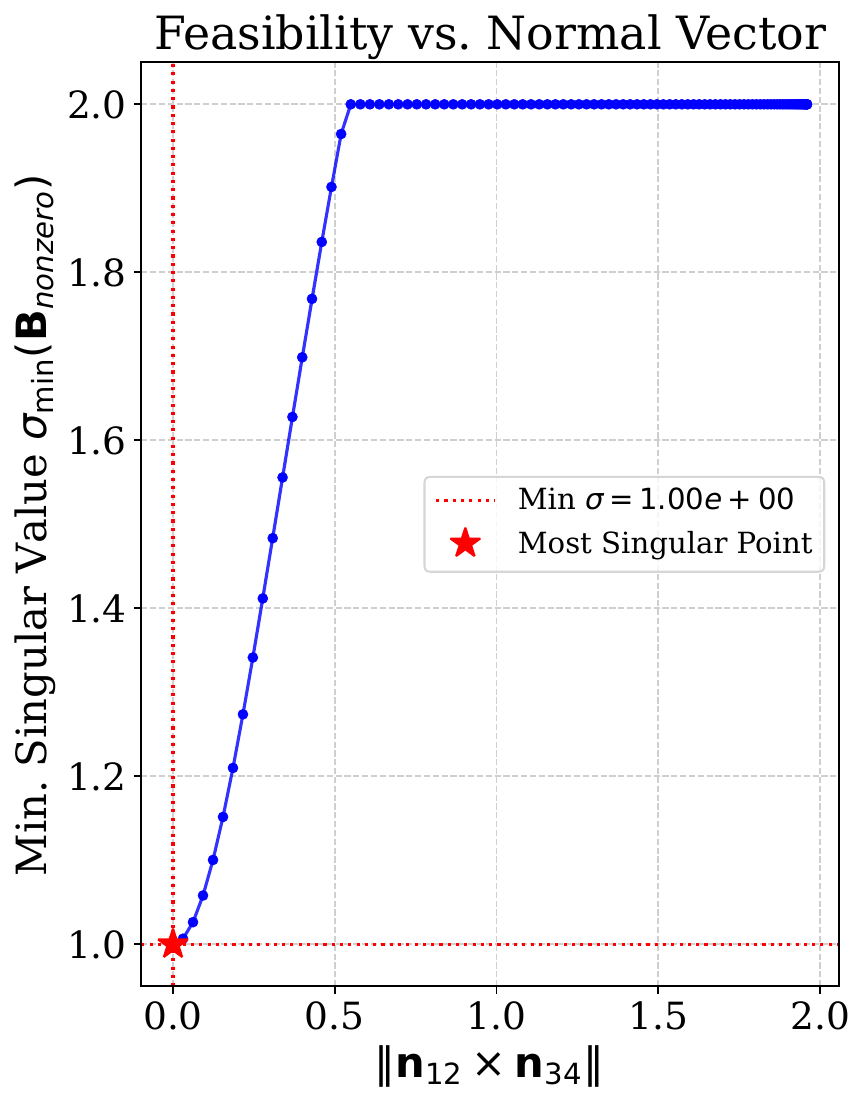}
            }
            \caption{The feasibility test of the nonzero part of the input matrix $\boldsymbol{B}_{nonzero}$ with different configurations of the normal vectors $\boldsymbol{n}_{12}$ and $\boldsymbol{n}_{34}$.}
            \label{fig:svd-B}
        \end{figure}
    }

    \subsection{Configuration control through CLF-HOCBF-QP}{
        Recognizing the physical interpretation of the natural equilibrium of $\boldsymbol{p}$, we devise the controller to drive the system configuration $\left(\boldsymbol{p}, \boldsymbol{n}_{12}, \boldsymbol{n}_{34}, d_{12}, d_{34}\right)$ to stabilize at a reference $\left(\boldsymbol{p}^{r}, \boldsymbol{n}_{12}^{r}, \boldsymbol{n}_{34}^{r}, d_{12}^{r}, d_{34}^{r}\right)$, where the major axis lengths are $d_{12} = \|\boldsymbol{p}_{1} - \boldsymbol{p}_{2}\|$ and $d_{34} = \|\boldsymbol{p}_{3} - \boldsymbol{p}_{4}\|$ of the two ellipsoids, respectively. 
        
        To address the feasibility issue discussed in Sec.~\ref{sec:feasibility}, we find desired robot positions $\boldsymbol{p}_i^{r}$ for $i=1, \dots, 4$ that satisfy ~\eqref{eq:pinandstring} by assuming symmetry, i.e., both triangles formed by a pair of ellipsoid foci and the hitch are isosceles, and $\boldsymbol{p}^{r}, \boldsymbol{p}_i^{r}$ are coplanar for $i=1, \dots, 4$ in $n = 3$,
        \begin{equation}
            \boldsymbol{p}_i^{r} = \boldsymbol{p}^{r} + \frac{d_{i}^{r}}{2}\left(\frac{-\boldsymbol{n}^{r}_{i}}{\sqrt{4-\|\boldsymbol{n}^{r}_{i}\|^2}} + (-1)^i \boldsymbol{\kappa}_{i}\times\boldsymbol{\hat n}^{r}_{i}\right),
            \label{eq:robot-position-ref}
        \end{equation}
        where $\boldsymbol{n}_{1}^{r} = \boldsymbol{n}_{2}^{r} = \boldsymbol{n}^{r}_{12}, \boldsymbol{n}^{r}_{3} = \boldsymbol{n}^{r}_{4} = \boldsymbol{n}^{r}_{34}$, similarly for $\boldsymbol{\hat n}^{r}_i$. $\boldsymbol{\kappa}_{1} = \boldsymbol{\kappa}_{2} = \frac{\boldsymbol{r}_1\times\boldsymbol{r}_2}{\|\boldsymbol{r}_1\times\boldsymbol{r}_2\|}, \boldsymbol{\hat n}^{r}_{12} = \frac{\boldsymbol{n}^{r}_{12}}{\|\boldsymbol{n}^{r}_{12}\|}, \boldsymbol{\kappa}_{3} = \boldsymbol{\kappa}_{4} = \frac{\boldsymbol{r}_3\times\boldsymbol{r}_4}{\|\boldsymbol{r}_3\times\boldsymbol{r}_4\|}, \boldsymbol{\hat n}^{r}_{34} = \frac{\boldsymbol{n}^{r}_{34}}{\|\boldsymbol{n}^{r}_{34}\|}.$

        As a result, the control objective is to find input $\boldsymbol{u}$ that drives the robots to positions $\boldsymbol{p}_i^{r}$ so that the cable-suspended system follows the reference configuration.
        
        Since the system dynamics~\eqref{eq:affine} is control-affine, an efficient and effective choice of control under these constraints is through the Control Lyapunov Function (CLF)-Control Barrier Function (CBF)-Quadratic Programming (QP) framework~\cite{sontag1983lyapunov,7782377}. We first define the Lyapunov function that compiles the objectives depending on the current state $\boldsymbol{x}$, 
        \begin{equation}
            V\left(\boldsymbol{x}\right) = \frac{1}{2}\sum_{i=1}^4 \boldsymbol{e}_{i}^\top\boldsymbol{K_p}\boldsymbol{e}_{i},
            \label{eq:control-lyapunov}
        \end{equation}
        where $\boldsymbol{e}_{\boldsymbol{p}_i}$ is the tracking error for robot $i$, and $\boldsymbol{K_p}$ is a user-defined, diagonal, and positive-definite gain matrix. 
        
        Ideally, the design of the error $\boldsymbol{e}_{i}$ should be such that the convergence of $V(\boldsymbol{x})\rightarrow0$ leads to robots stabilizing at positions $\boldsymbol{p}_i^{r}$, which further induces that the system achieves the reference configuration. Since the achieving desired robot position is a sufficient condition for the system to achieve reference configuration as $\boldsymbol{p}_i^{r}$ is a direct derivation from the reference configuration, we need to ensure that convergence of the Lyapunov function $V(\boldsymbol{x})$ results in the convergence of the robot position.
        Thus, the design of the errors directly impacts the feasibility of maintaining a converging CLF, including ensuring a relative degree of 1 over the input $\boldsymbol{u}$, which we will discuss after the construction of the optimization problem.

        In addition to the converging Lyapunov function, the control input $\boldsymbol{u}$ of the cable-suspended system must be such that satisfies realistic constraints, including positive tension force $\boldsymbol{t}\succ\boldsymbol{0}$, and the cable segment length constraints $q_i(\boldsymbol{x}) > 0$, where $q_i(\boldsymbol{x}) = l_{12} - \|\boldsymbol{r}_i\|$ for $i = 1, 2$ and $q_i(\boldsymbol{x}) = l_{34} - \|\boldsymbol{r}_i\|$ for $i = 3, 4$. However, $q_i(\boldsymbol{x})$ has a relative degree of $2$ to the input vector $\boldsymbol{u}$, i.e., only by taking the second-order derivative of $q_i(\boldsymbol{x})$ can $\boldsymbol{u}$ be revealed with the acceleration of the robots and the hitch. To handle this, we introduce High-Order Control Barrier Functions (HOCBF) by defining a new constraint $\psi_i(\boldsymbol{x}) = \dot q_i(\boldsymbol{x}) + \beta q_i(\boldsymbol{x}) \geq 0$, where $\beta$ is a user-defined parameter that specifies the resistance to the cable segment length from approaching the total cable length. 
        
        Considering these constraints, we obtain the optimal input~$\boldsymbol{u}$ by solving the quadratic programming problem,
        \begin{align}
            & \underset{\boldsymbol{u}, \delta}{\text{min}} & & (\boldsymbol{u-u}_{nominal})^\top(\boldsymbol{u-u}_{nominal}) + \alpha\delta^2,\label{eq:objective}\\
            & \text{subject to} & & \dot V(\boldsymbol{x}) + \gamma V(\boldsymbol{x})\leq\delta\label{eq:clf}\\
            & & & \dot \psi_i(\boldsymbol{x}) + \lambda \psi_i(\boldsymbol{x}) \geq 0, i = 1\dots, 4,\label{eq:cbf}\\
            & & & \boldsymbol{t}\succeq t_{min},\label{eq:linear-input-constraints}\\
            & & & -f_{max}\preceq\boldsymbol{u}\preceq f_{max}.\label{eq:input-constraint}
        \end{align}
        where $\alpha\gg 0$ is a user-defined large positive scalar that penalizing the reference tracking divergence, $\delta\gtrapprox 0$ is a slack variable that measures the reference tracking convergence violation, $\gamma$ is a user-defined parameter to specify the minimum exponential convergence rate of the Lyapunov function, and $\lambda$ is a user-defined parameter to specify the minimum exponential divergence rate of the barrier function. In terms of the aforementioned variables, $\psi_i = -\boldsymbol{\hat r}_i\boldsymbol{\dot r}_i + \beta(l_{i} - \|\boldsymbol{r}_i\|)$, where $l_{i} = l_{12}$ for $i = 1, 2$ and $l_{34}$ for $i = 3, 4.$ $t_{min}$ is a user-defined parameter for the minimum tension along the cable, and the nominal input $\boldsymbol{u}_{nominal} = -t_{min}[\boldsymbol{\hat r}_1^\top, \boldsymbol{\hat r}_2^\top, \boldsymbol{\hat r}_3^\top, \boldsymbol{\hat r}_4^\top]^\top$ following the discussion in Sec.~\ref{sec:nominal}.
        
        The introduction of the reference tracking divergence penalty in the CLF~\eqref{eq:clf} allows the system to treat the Lyapunov function convergence as a ``soft'' constraint than the safety-critical constraints, including the CBF~\eqref{eq:cbf} and the positive tension constraint~\eqref{eq:linear-input-constraints}. We note that although an input-dependent constraint~\eqref{eq:linear-input-constraints} is not common in CBF approaches~\cite{agrawal2021safe}, the fact that our positive tension constraint is linear to the input $\boldsymbol{u}$ makes the constraint compatible with the QP formulation without further manipulation. We expand the input-affine constraints~\eqref{eq:clf}~\eqref{eq:cbf}~\eqref{eq:linear-input-constraints} to obtain the final formulation of the Quadratic Programming Problem,
        \begin{equation}
            \begin{aligned}
                & \underset{\boldsymbol{u}, \delta}{\text{min}} & & (\boldsymbol{u-u}_{nominal})^\top(\boldsymbol{u-u}_{nominal}) + \alpha\delta^2,\\
                & \text{subject to} & & L_{\boldsymbol{h}}V + L_{\boldsymbol{B}}V\boldsymbol{u} + \gamma V\leq\delta,\\
                & & & L_{\boldsymbol{h}}\psi_i + L_{\boldsymbol{B}}\psi_i\boldsymbol{u} + \lambda\psi_i \geq 0, i = 1, \dots, 4,\\
                & & & \boldsymbol{M}^{-1}\boldsymbol{Cu}+\boldsymbol{M}^{-1}\boldsymbol{w}\succ 0,\\
                & & & -f_{max}\preceq\boldsymbol{u}\preceq f_{max},
            \end{aligned}
            \label{eq:clf-hocbf-qp}
        \end{equation}
        where $L_{\text{vec}}(\text{func}) = \frac{\partial}{\partial\boldsymbol{x}}\text{func}(\boldsymbol{x})\cdot\text{vec}(\boldsymbol{x})$ represents the Lie derivative of a function along a vector field, and in the control implementation, the drift~\eqref{eq:drift} field does not consider the unknown damping and external forces $\boldsymbol{f}_d, \boldsymbol{f}_{ext}$.
    }
    \subsection{Error design and convergence}{
        
        A crucial criteria to ensure the feasibility of the QP~\eqref{eq:clf-hocbf-qp} is that both the Lyapunov function $V(\boldsymbol{x})$ and the Barrier function $\psi_i(\boldsymbol{x})$ have a relative degree of 1 with respect to the input $\boldsymbol{u}$, which means both $L_{\boldsymbol{B}}V\neq\boldsymbol{0}$ and $L_{\boldsymbol{B}}\psi_i\neq\boldsymbol{0}$. It is not non-trivial since our system is a second-order system, i.e., $\boldsymbol{u}$ affects the acceleration of the objects. For the cable segment length constraints, we use HOCBF to ensure $L_{\boldsymbol{B}}\psi_i\neq\boldsymbol{0}$. 
        
        If we directly define the desired robot positions, and compute the errors respectively as the difference between the reference and the current value, $V(\boldsymbol{x})$ will have a relative degree of 2 with respect to $\boldsymbol{u}$.
        Although adopting an Higher-Order Control Lyapunov Function (HOCLF) approach could solve the issue of the high relative degree, i.e., replacing~\eqref{eq:clf} with $\ddot V(\boldsymbol{x}) + c_1\dot V(\boldsymbol{x}) + c_0 V(\boldsymbol{x})\leq 0$ and choose the coefficients $c_1, c_0$ such that the characteristic polynomial of the second-order derivative is Hurwitz~\cite{5991573,7524935,chriat2024high}, the second-order derivative can amplify the unmodeled disturbance and noises, including but not limited to $\boldsymbol{f}_{d}$ and $\boldsymbol{f}_{ext}$. Unlike the HOCBF $\psi_i$, of which the Lie-derivatives are sparse as shown in~\eqref{eq:hocbf-comp}, $\ddot V$ will depend on all state variables, which makes its convergence highly sensitive to the noises.  

        Instead, we define the errors of each corresponding quantity as the composite of its reference velocity error and cascaded position error~\cite{khan2020barrier}. Specifically, 
        \begin{equation}
            \boldsymbol{e}_{i} = \boldsymbol{v}^{r}_{i} + \boldsymbol{K}^{cas}_{\boldsymbol{p}}\boldsymbol{e}^{cas}_{\boldsymbol{p}_i} - \boldsymbol{\dot p}_i, i=1, \dots, 4,
        \end{equation}
        where the cascade errors $\boldsymbol{e}^{cas}_{\boldsymbol{p}_i} = \boldsymbol{p}_i^{r} - \boldsymbol{p}_i$ is the position error of robot $i$,
        and $\boldsymbol{v}^{r}_{i}$ is reference velocity of robot $i$. $\boldsymbol{K}^{cas}_{\boldsymbol{p}}$ is the positive-definite diagonal gain matrix.
        In this design, we control the robot positions to follow the references, similar to a velocity control, which converts them into an additional term to the desired velocity reference. Since the Lyapunov function depends on the state variables of the latter $5n$ elements, its partial derivative over the state vector will have nonzero latter half of $5n$ elements, of which the multiplication with the input matrix $\boldsymbol{B}$ is nonzero in $L_{\boldsymbol{B}} V$, ensuring the authority of the input vector $\boldsymbol{u}$ in satisfying the CLF constraint in the QP~\eqref{eq:clf-hocbf-qp}.

        \begin{proposition}
            Convergence of the Lyapunov function $V(\boldsymbol{x})\rightarrow0$ guarantees the convergence of the robot positions $\boldsymbol{p}_i\rightarrow\boldsymbol{p}_i^{r}$ for $i=1,\dots 4$.
            \label{prop:convergence-pos}
        \end{proposition}

        \begin{proof}
            With the slack variable $\delta = 0$, the CLF constraint~\eqref{eq:clf} becomes $\dot V(\boldsymbol{x}) + \gamma V(\boldsymbol{x})\leq0$. As $t\rightarrow\infty$, $V(\boldsymbol{x})\rightarrow 0$ constrained by the CLF~\eqref{eq:clf}.  Since $V(\boldsymbol{x})$ is a sum of squared norms of the composite errors $\boldsymbol{e}_i$, $V(\boldsymbol{x})\rightarrow 0$ implies that for each robot, $\boldsymbol{e}_i\rightarrow0$, which expands to 
            \begin{equation}
                \boldsymbol{v}^{r}_{i} + \boldsymbol{K}^{cas}_{\boldsymbol{p}}\boldsymbol{e}^{cas}_{\boldsymbol{p}_i} - \boldsymbol{\dot p}_i\rightarrow0.
                \label{eq:vel-error-convergence}
            \end{equation}
            Taking derivative of the cascade position errors, we obtain
            \begin{align}
                \boldsymbol{\dot e}^{cas}_{\boldsymbol{p}_i} =& \frac{d}{dt}(\boldsymbol{p}_i^{r} - \boldsymbol{p}_i)\\
                =& \boldsymbol{v}_{i}^{r}-\boldsymbol{\dot p}_i.\label{eq:proof-cascade-error-convergence}
            \end{align}
            Given~\eqref{eq:vel-error-convergence}, we further expand~\eqref{eq:proof-cascade-error-convergence},
            \begin{equation}
                \boldsymbol{\dot e}^{cas}_{\boldsymbol{p}_i} \rightarrow \boldsymbol{v}_{i}^{r}-\boldsymbol{v}^{r}_{i} - \boldsymbol{K}^{cas}_{\boldsymbol{p}}\boldsymbol{e}^{cas}_{\boldsymbol{p}_i},
            \end{equation}
            which shows $\boldsymbol{\dot e}^{cas}_{\boldsymbol{p}_i} + \boldsymbol{K}^{cas}_{\boldsymbol{p}}\boldsymbol{e}^{cas}_{\boldsymbol{p}_i}\rightarrow\boldsymbol{0}$. Since $\boldsymbol{K}^{cas}_{\boldsymbol{p}}$ is a positive-definite diagonal matrix, the robot position error as a cascade term in the composite velocity error, $\boldsymbol{e}^{cas}_{\boldsymbol{p}_i}$ converge to $0$ as the CLF constraint is satisfied, which means the robot achieves desired positions.
        \end{proof}
        Once the robot satisfies the desired position, the hitch position $\boldsymbol{p}$, ellipsoid normal vectors $\boldsymbol{n}_{12}$ and $\boldsymbol{n}_{34}$, and ellipsoid major axis lengths $d_{12}, d_{34}$ will converge at the desired references, as given in~\eqref{eq:robot-position-ref}.
    }

}
\begin{figure*}[t]
    \centering
    \subfloat[{The evolution of the Lyapunov function\label{fig:L-3D-static}}]{
        \centering
        \includegraphics[width=0.45\linewidth]{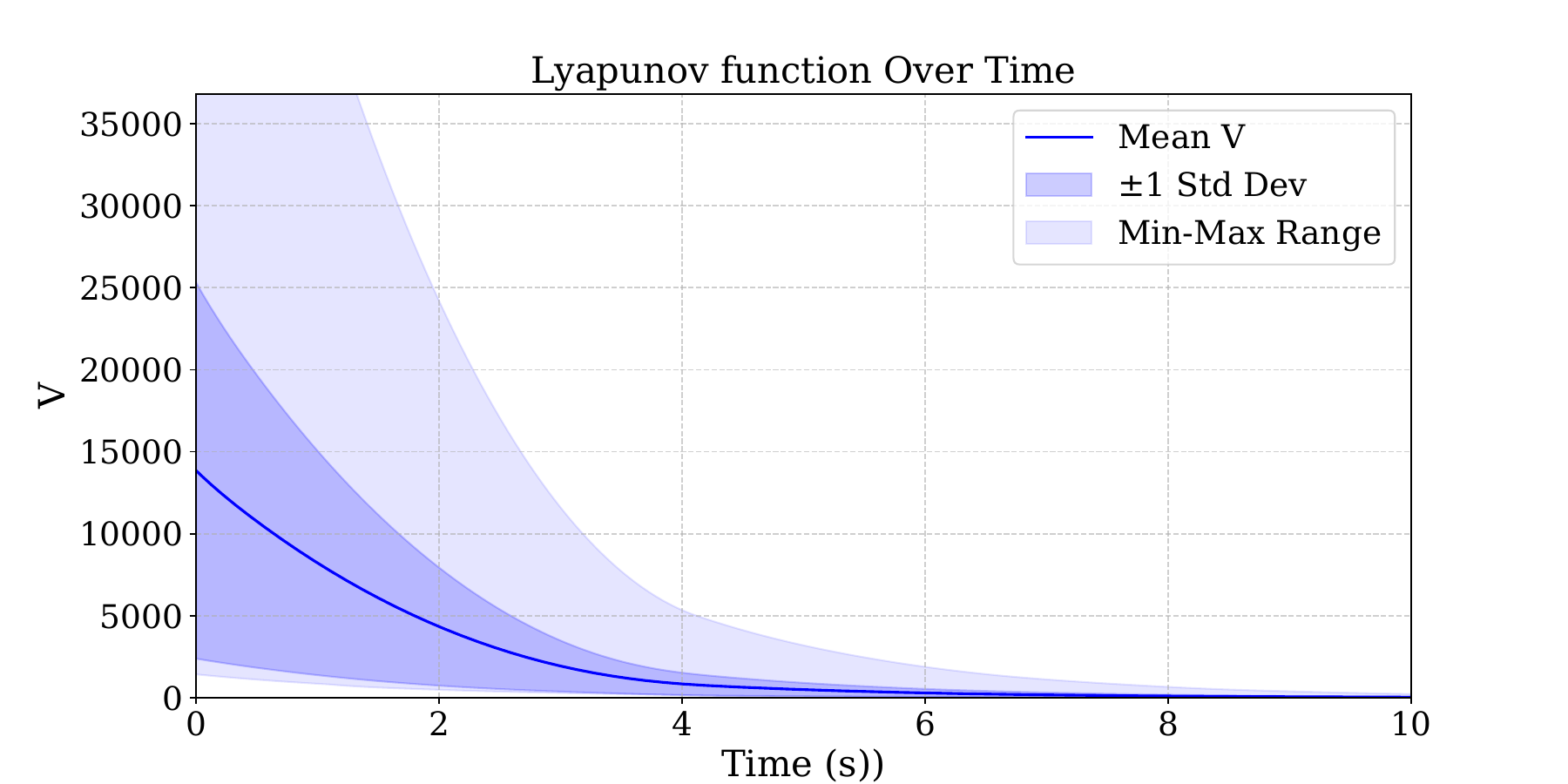}
    }\quad
    \subfloat[{The evolution of the reference error summation\label{fig:E-3D-static}}]{
        \centering
        \includegraphics[width=0.45\linewidth]{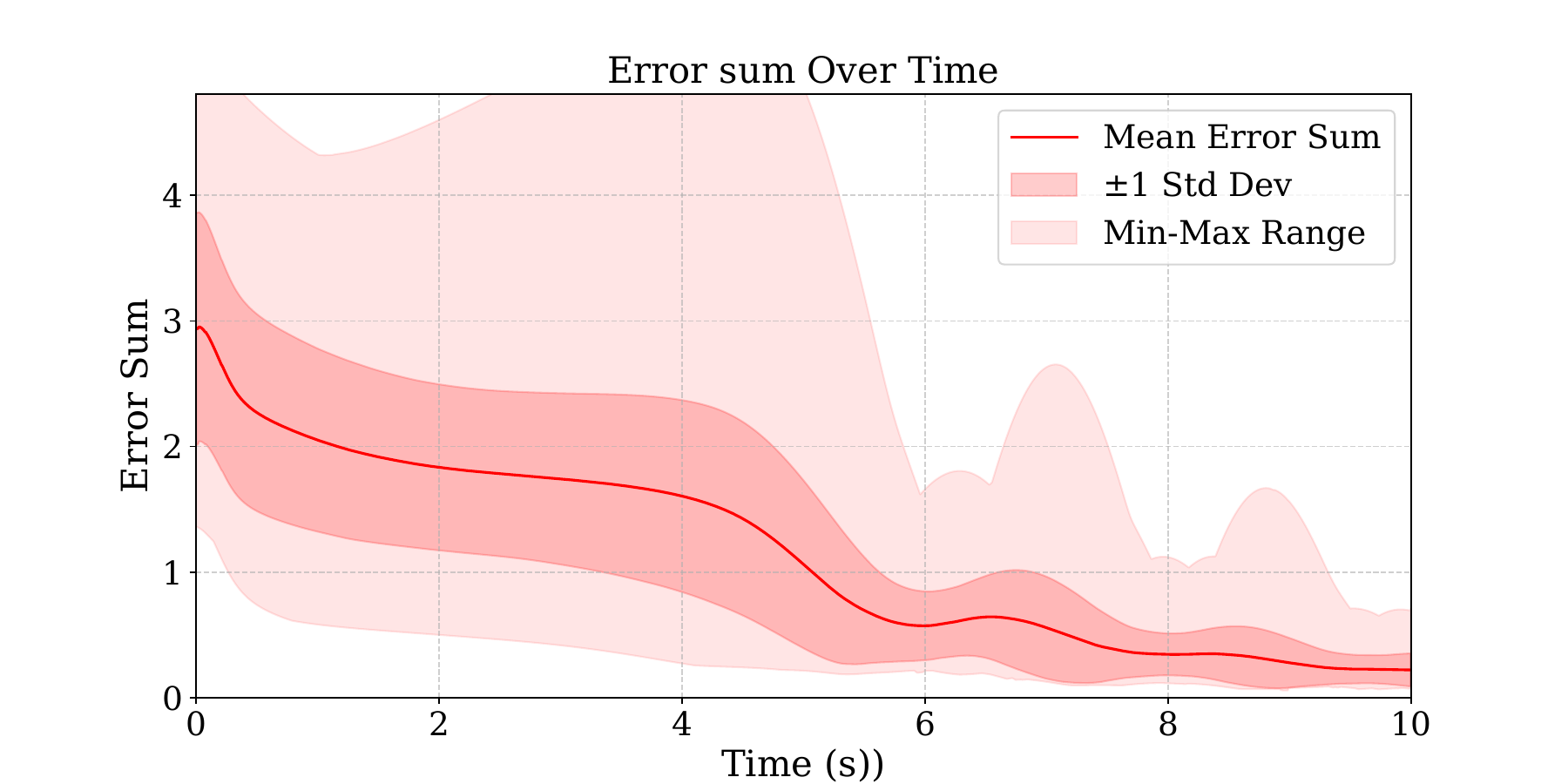}
    }

    \caption{Experiment 1: the cable-suspended aerial system in 3D following a static reference configuration from a random initialization, with first-order damping $c_d = 0.2$ at the hitch unknown to the controller.}
    \label{fig:3d-plots-static}
\end{figure*}

\begin{figure*}[t]
    \centering
    \subfloat[{The evolution of the Lyapunov function\label{fig:L-3D-noisy}}]{
        \centering
        \includegraphics[width=0.45\linewidth, height=0.2\linewidth]{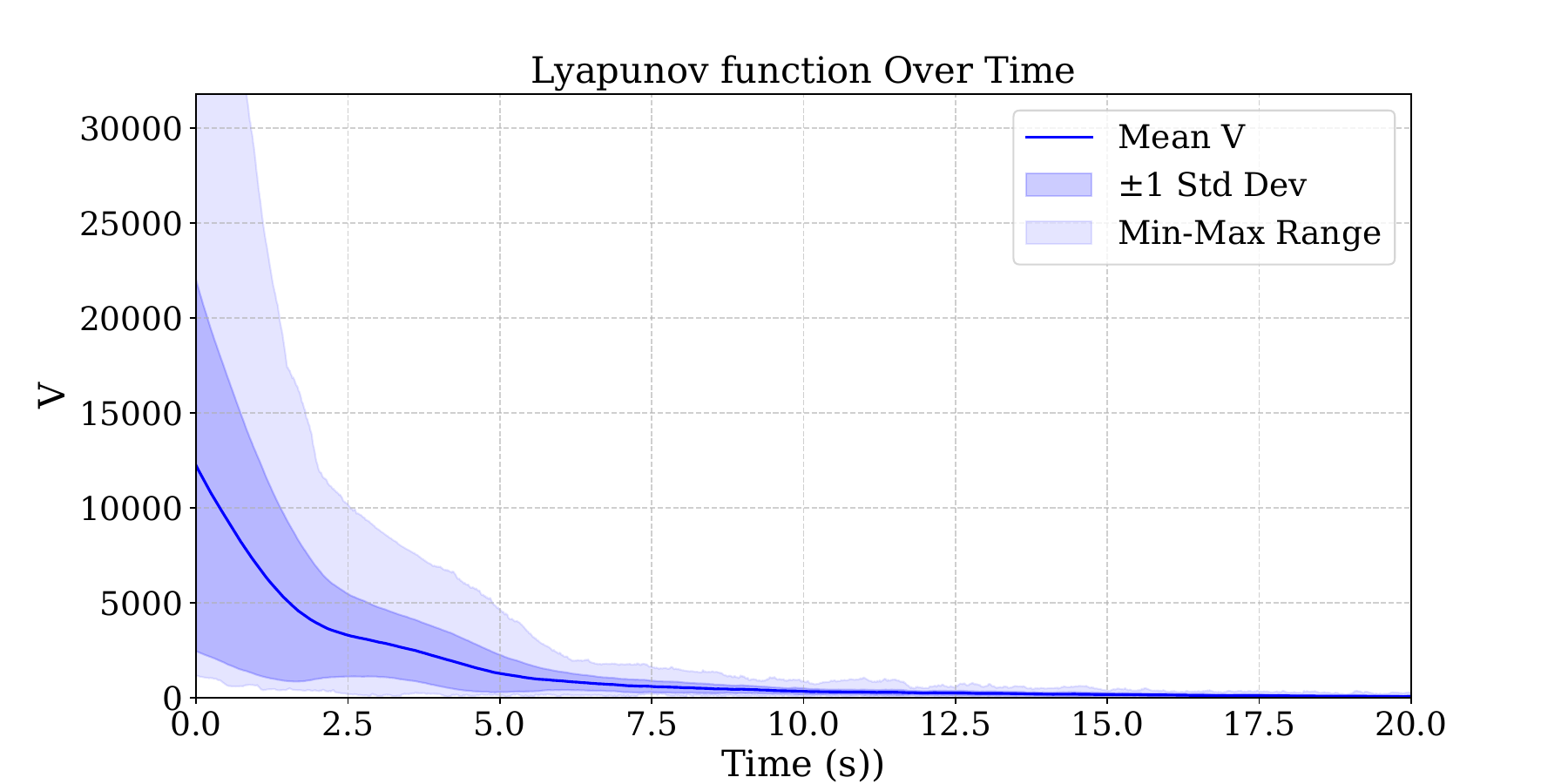}\qquad
    }\quad
    \subfloat[{The evolution of the reference error summation\label{fig:E-3D-noisy}}]{
        \centering
        \includegraphics[width=0.45\linewidth, height=0.2\linewidth]{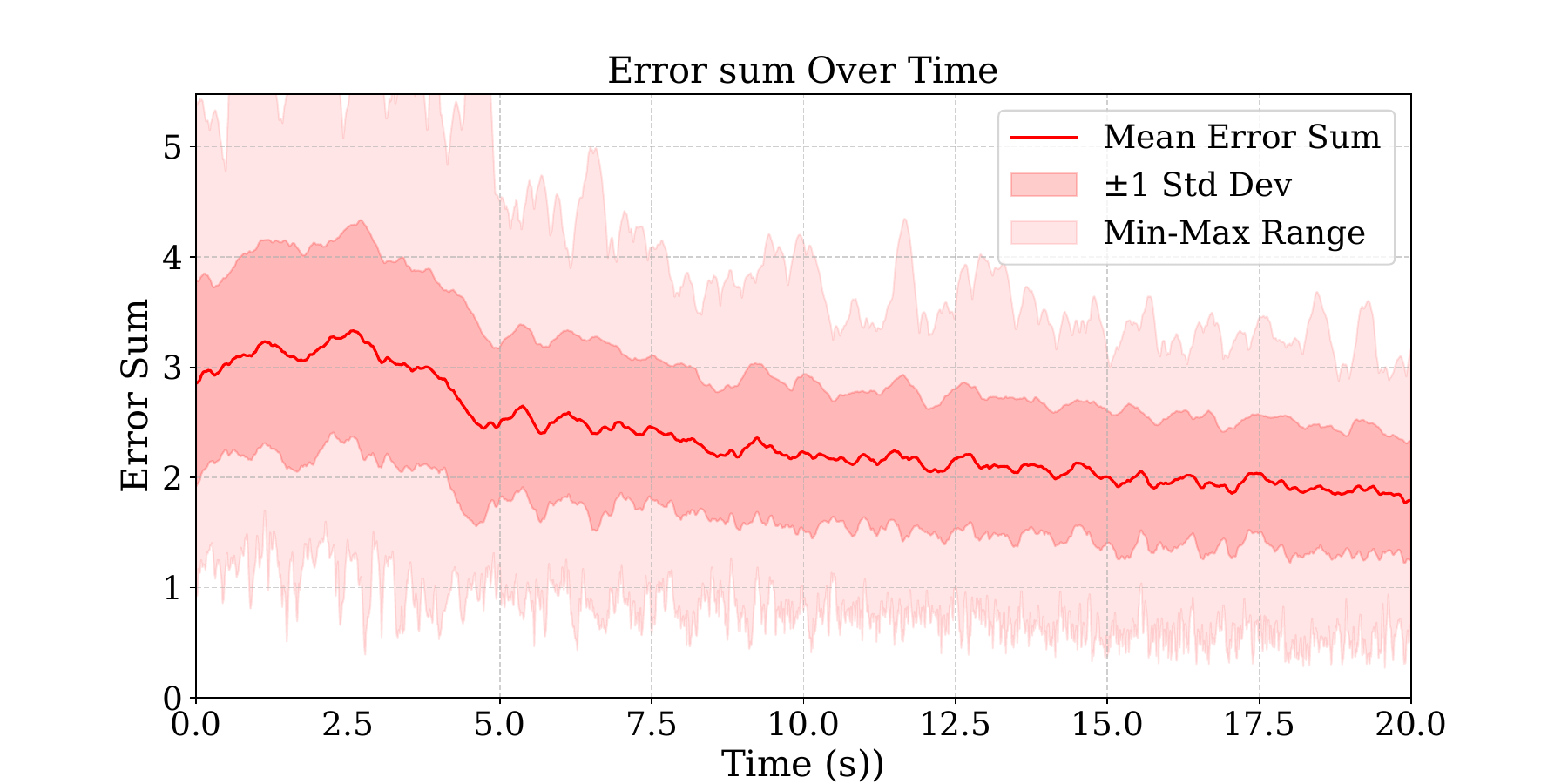}
    }
    \caption{The results of the system in 3D following a dynamic reference configuration of which the hitch moves at $0.1$m/s from a random initial state, with $c_d = 0.2$ and a Gaussian noise with a standard deviation of $5$N at the hitch both unknown to the controller.}
    \label{fig:3d-plots-noisy}
\end{figure*}
\section{Evaluation}{
    We validate the proposed cascade CLF-HOCBF-QP framework in numerical simulations implemented in Python for both 2-dimensional ($n=2$) and 3-dimensional ($n=3$) evaluations, although we present only 3D evaluations in this paper. First, we implement the dynamics of the cable-suspended system described by the control-affine differential equation~\eqref{eq:2-cable-4-robot-system} using Euler integration in discrete time increments $\Delta t$. To minimize the drift caused by the numerical integration, we project the integration onto the cable-induced constraints by solving~\eqref{eq:pinandstring} from the integrated values to ensure the efficacy of the model. At every time step, we construct the quadratic programming problem in~\eqref{eq:clf-hocbf-qp} using cvxpy~\cite{diamond2016cvxpy,agrawal2018rewriting} based on the current states, and solve it to obtain the optimal input $\boldsymbol{u}$ using the Operator Splitting Quadratic Programming solver~\cite{osqp}.~\footnote{The implementation is available at~\url{https://github.com/Jarvis-X/PinandStringEllipse}}.

    We choose the simulation parameters and gain coefficients as follows, $\Delta t = 0.005$s, $\boldsymbol{K_p} = 20\boldsymbol{I}, \boldsymbol{K_p}^{cas} = 10\boldsymbol{I}$, $\gamma=100\Delta t, \alpha=10^6, \beta=10, \lambda=100$, $t_{min} = 0.1$N, $u_{max} = 20$N, $m_i = 0.35$kg, $m=0.005$kg. Under these parameters, solving the QP typically takes $<0.01s$ at each step running on an Intel i7-11800H low-voltage CPU, which highlights the real-time practicality of the controller.

    In the spatial case $n = 3$, we perform three sets of experiments. First, we randomly initialize the state variables that satisfy the kinematic model~\eqref{eq:pinandstring}, and let the controller drive the system to follow a static reference configuration. Second, we randomly initialize the state variables, and let the controller drive the system dynamics to follow a slow reference while the hitch is receiving a large Gaussian noise. Third, we gradually increase the reference speed, and let the system to follow the reference dynamically. We report the statistics in terms of the mean, standard deviation, and min-max values of the Lyapunov function value and the summation of reference errors. Notably, in all three experiments, we introduce a first-order damping at the hitch $\boldsymbol{f}_d = -c_d\boldsymbol{\dot p}$, which is unknown to the controller.
    
    Each set of simulation is repeated $100$ times with different initial configurations, and we record the evolution of the Lyapunov function~\eqref{eq:control-lyapunov} and the summation of error norms over the course of each simulation,
    \begin{equation}
        \begin{aligned}
            e_{sum} =& \|\boldsymbol{p}^{r} - \boldsymbol{p}\| + \|\boldsymbol{n}_{12}^{r} - \boldsymbol{n}_{12}\| \\
            &+ \|\boldsymbol{n}_{34}^{r} - \boldsymbol{n}_{34}\| + |d_{12}^{r} - d_{12}| + |d_{34}^{r} - d_{34}|,
        \end{aligned}
    \end{equation}
    defined in terms of the difference between the current configuration of the cable-suspended system and a reference. 
    
    For experiment 1, we report the statistics over a period of $10$ seconds. For experiment 1, we report the statistics over a period of $20$ seconds. For experiment 3, we run the simulation for $15$s, and report the statistics after $10$ seconds for each reference speed. Across all experiments, despite the excessive numerical values of $V$ due to its weighted composition of position and velocity errors, its evolution over time always diminishes as time progresses.

    \paragraph{Experiment 1 - Static reference} 
    The first experiment validates the fundamental performance of the proposed control framework in a static tracking scenario. As depicted in Fig.~\ref{fig:L-3D-static}, the controller successfully drives the randomly initialized system to a fixed reference configuration. The exponential convergence of the Lyapunov function to zero, shown in Fig.~\ref{fig:3d-plots-static}, directly corresponds to the convergence of the composite velocity errors for each robot. As proven in Proposition~\ref{prop:convergence-pos}, this guarantees that the robot position errors also converge to zero. This is empirically confirmed in Fig.~\ref{fig:E-3D-static}, where the sum of configuration errors, $e_{sum}$, rapidly diminishes. This result confirms the efficacy of both the CLF-HOCBF-QP controller and the  mapping~\eqref{eq:robot-position-ref} used to derive target robot positions from the desired system configuration.
    
    \begin{figure*}[t]
        \centering
        \subfloat[{Terminal Lyapunov function statistics\label{fig:L-3D-stats}}]{
            \centering
            \includegraphics[width=0.45\linewidth]{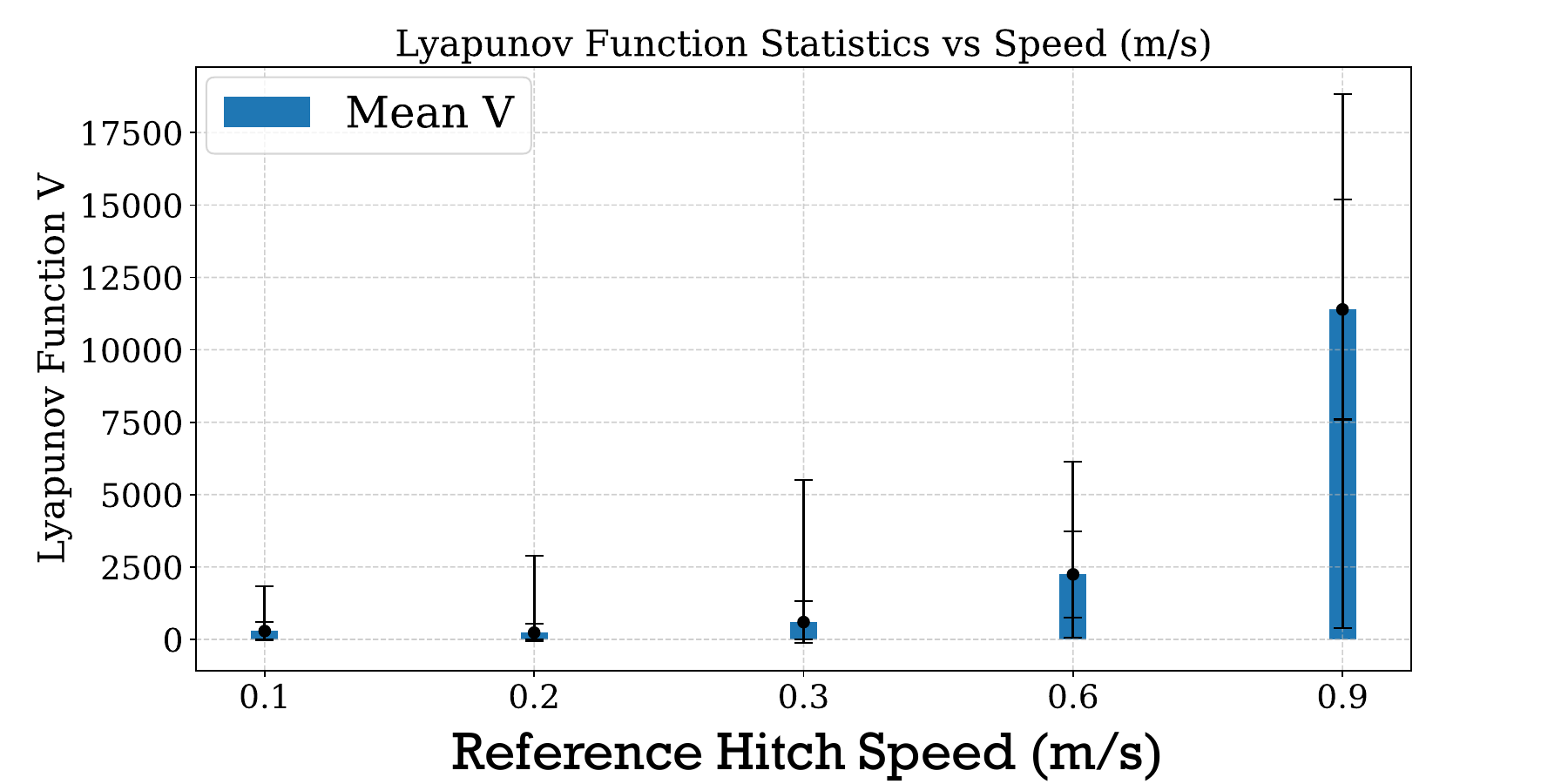}\qquad
        }\quad
        \subfloat[{Terminal reference error summation statistics\label{fig:E-3D-stats}}]{
            \centering
            \includegraphics[width=0.45\linewidth]{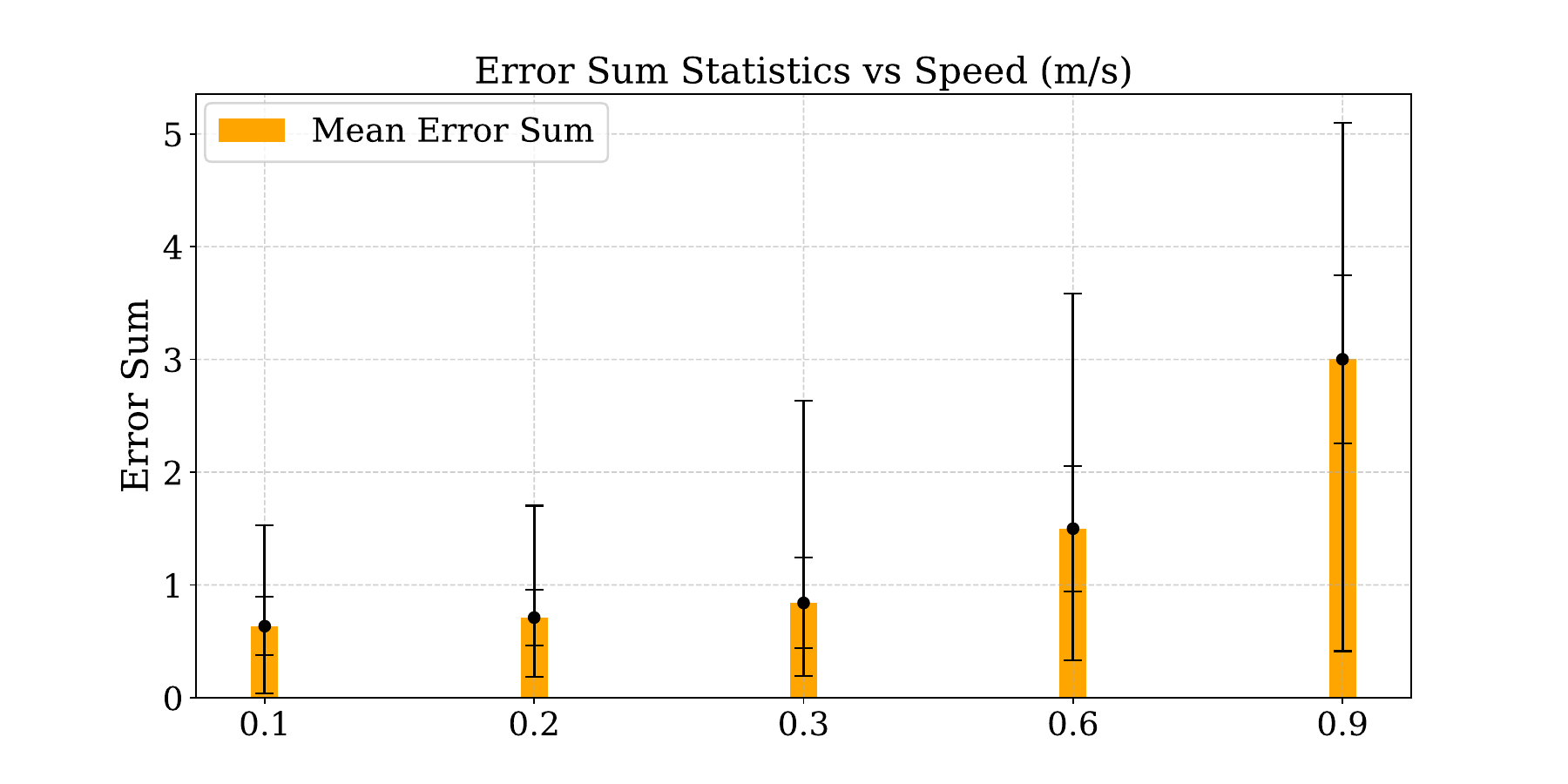}
        }
        \caption{The results of the system following a reference hitch moving at increasing speeds from a random initial state 3D, with $c_d = 0.2$. We report the mean, standard deviation, minimum, and maximum of the Lyapunov function value and the summation of system configuration errors relative to the reference, evaluated at $t>10$s into a $15$s simulation.}
        \label{fig:3d-plots-stats}
    \end{figure*}
    \paragraph{Experiment 2 - Noisy hitch} 
    This experiment evaluates the controller's robustness to significant, unmodeled external disturbances, a critical attribute for real-world viability. The system was tasked with tracking a slow-moving reference (hitch speed of $0.1$m/s) while the hitch was subjected to a zero-mean Gaussian force with a standard deviation of $5$N. Given the small virtual mass of the hitch (m=0.005 kg), this noise corresponds to a disturbance acceleration on the order of 1000 m/s$^2$, which is substantial relative to the system's dynamics and comparable to the control authority (maximum input of 20 N). Despite this large, unmodeled disturbance, the controller attains stability across all 100 trials, and the Lyapunov function consistently converges toward zero (Fig.~\ref{fig:L-3D-noisy}). While the persistent noise results in a non-zero configuration error (Fig.~\ref{fig:E-3D-noisy}), the system remains bounded and tracks the reference effectively. This demonstrates the controller's inherent robustness and its ability to reject high-magnitude external disturbances, a key challenge in practical aerial manipulation. 

    \paragraph{Experiment 3 - Dynamic reference} 
    The third experiment directly addresses the central claim of this paper: that the proposed dynamic control framework enables agile manipulation beyond the limits of quasi-static approaches. The system tracked a reference configuration of which the hitch followed a Lissajous curve at progressively increasing speeds, from $0.1$m/s to $0.9$m/s. Such speeds render quasi-static assumptions, along with associated control approaches to fail. In Fig.~\ref{fig:3d-plots-stats}, the terminal Lyapunov function value and configuration error increase with the reference speed, indicating that perfect tracking becomes more challenging as actuator limits are approached. However, two key observations highlight the success of our method. First, the system successfully maintains stability and tracks the reference across all tested speeds without violating any constraints. Achieving stable formation flight at speeds approaching $1$m/s surpasses the agility of existing methods for inter-cable manipulation. Second, even at the highest speed of $0.9$m/s, where the Lyapunov function stabilizes at a non-zero value, the configuration error remains bounded. This demonstrates the effectiveness of the HOCBF, which prioritizes safety-critical constraints over reference tracking. 
}

\section{Conclusion}{
    In this paper, we presented a novel framework for the dynamic control of a hitch formed by two interlacing cables manipulated by four aerial robots, introducing a control-affine dynamic model derived from ellipsoid geometry and a CLF-HOCBF-QP controller to stabilize the system to a reference configuration. We addressed the high relative degree of the Lyapunov function by integrating a cascade term into the composite error, and proved that the convergence of the Lyapunov function leads to the convergence of the configuration to a reference. Our simulations validated our theoretical results on stable tracking of dynamic trajectories, validating the effectiveness of our approach. As the model and control have been implemented and verified in simulation, we plan to test the system on a physical hardware platform to evaluate its real-world performance. A key challenge to address in future work is the uncertainty of the virtual mass parameter, which emulates the hitch's resistance to motion which requires real-time adaptation to estimate this unknown quantity, enhancing the controller's robustness and applicability.
}

\bibliographystyle{IEEEtran}
\bibliography{ref.bib}
\begin{figure}
    \centering
    \subfloat[{Terminal Lyapunov function statistics\label{fig:L-3D-fast}}]{
        \centering
        \includegraphics[width=\linewidth]{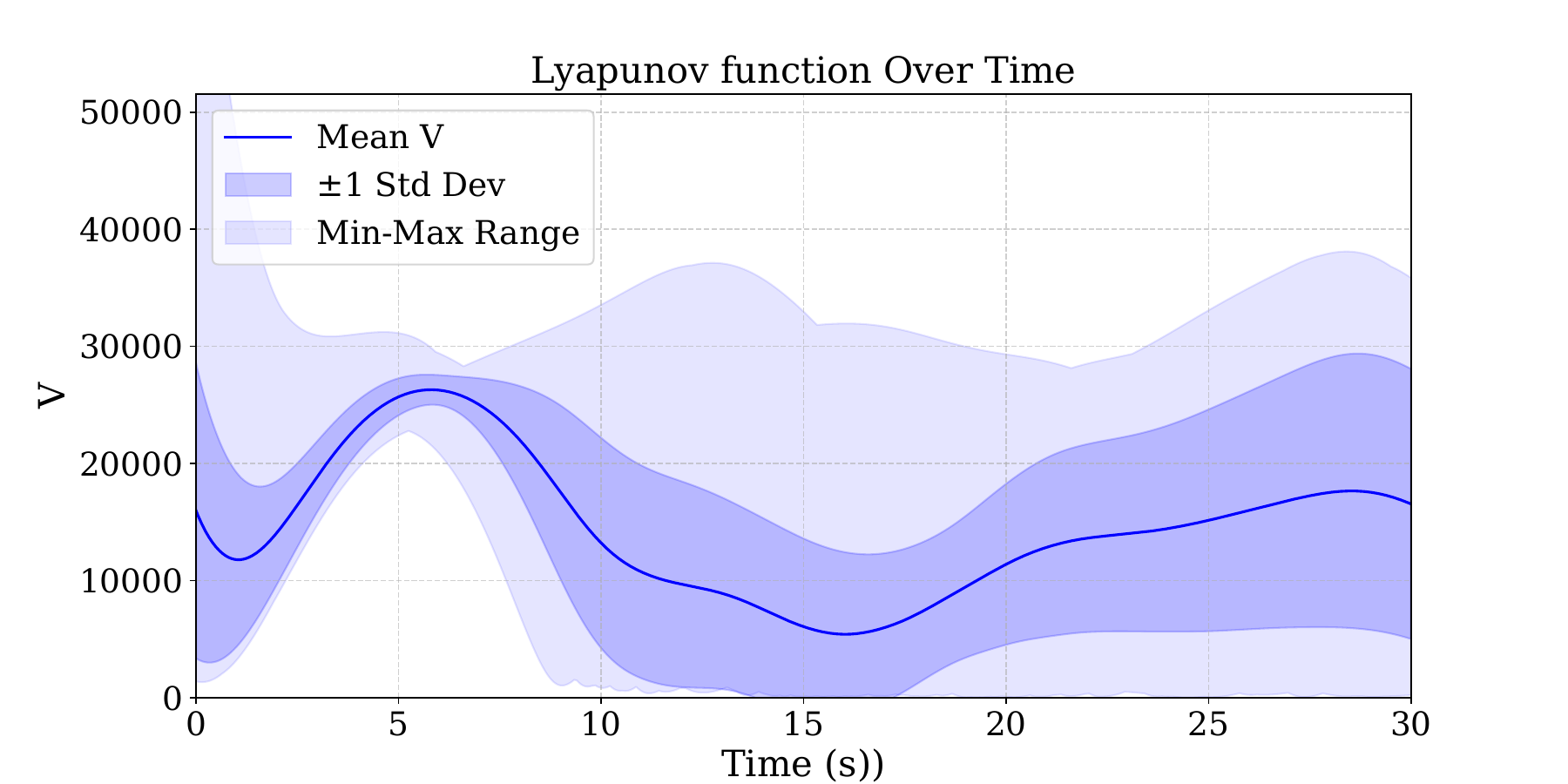}\qquad
    }\quad
    \subfloat[{Terminal reference error summation statistics\label{fig:E-3D-fast}}]{
        \centering
        \includegraphics[width=\linewidth]{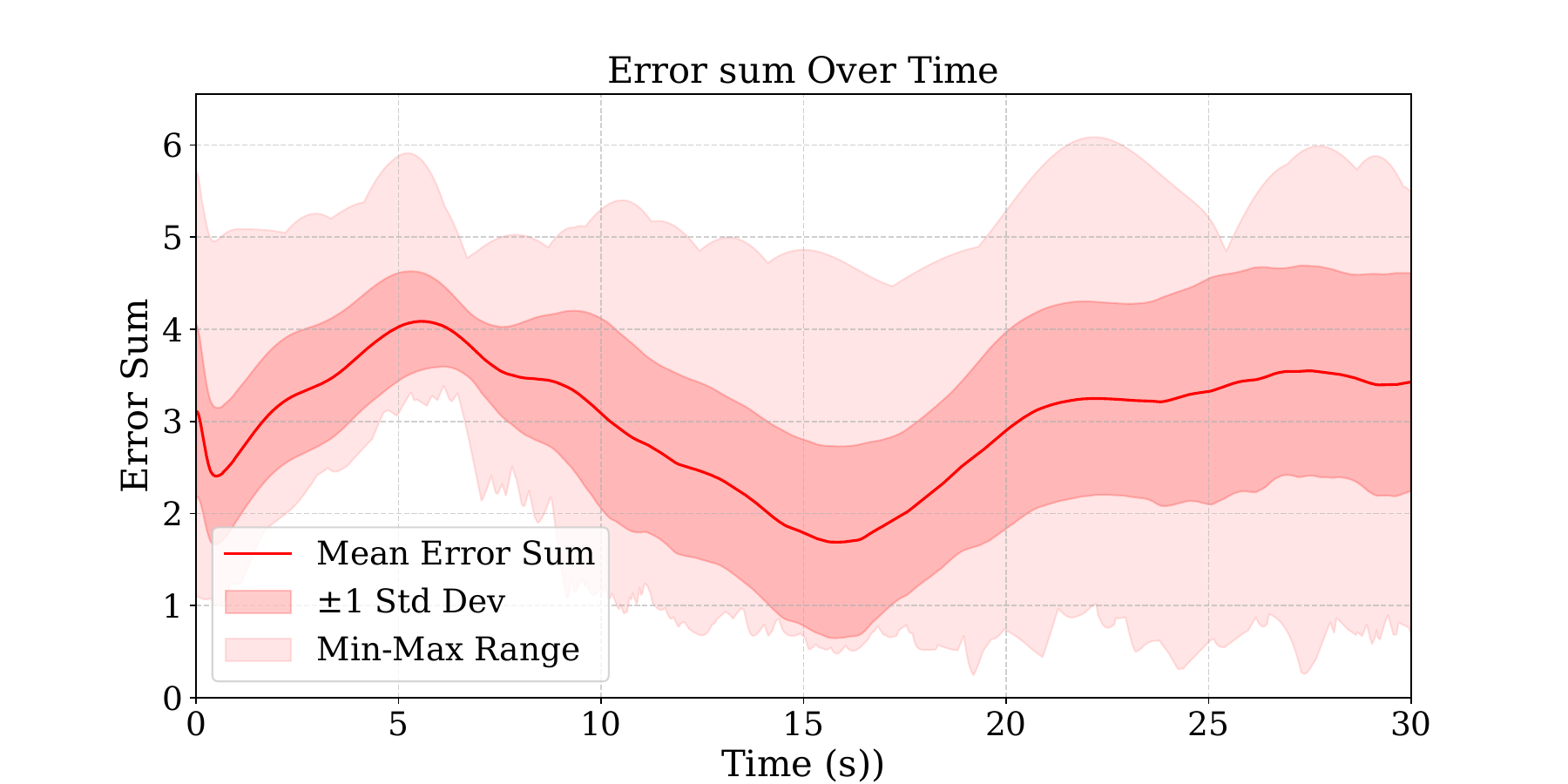}
    }
    \caption{The results of the system in 3D following a reference configuration of which the hitch moving at $0.9$m/s from a random initial state, with $c_d = 0.2$, unknown to the controller.}
    \label{fig:3d-plots-fast}
\end{figure}

\appendix{
    \subsection{Partial Derivatives}{
        The Lie-derivatives of $V(\boldsymbol{x})$~\eqref{eq:control-lyapunov} involves computing the gradient of $V$ over the state vector $\boldsymbol{x}$. Therefore, we show the partial derivative components of the gradient. 
        \begin{equation}
            \begin{aligned}
                \frac{\partial V}{\partial \boldsymbol{p}} &= \boldsymbol{0}, &\frac{\partial V}{\partial \boldsymbol{p}_i} &= -\boldsymbol{K_p}^{cas}\boldsymbol{K_p}\boldsymbol{e}_i\\
                \frac{\partial V}{\partial \boldsymbol{\dot p}} &= \boldsymbol{0}, &\frac{\partial V}{\partial \boldsymbol{\dot p}_i} &= -\boldsymbol{K_p}\boldsymbol{e}_i\\
            \end{aligned}
        \end{equation}

        Similarly for the Barrier function $\psi_i(\boldsymbol{x})$, the partial derivative components are
        \begin{equation}
            \begin{aligned}
                \frac{\partial \psi_i}{\partial \boldsymbol{p}} =& -\beta\boldsymbol{\hat r}_i^\top - \frac{1}{\|\boldsymbol{r}_i\|}(\boldsymbol{\dot p} - \boldsymbol{\dot p}_i)^\top(\boldsymbol{I} - \boldsymbol{\hat r}_i\boldsymbol{\hat r}^\top_i)\\ 
                \frac{\partial \psi_i}{\partial \boldsymbol{p}_i} =& \beta\boldsymbol{\hat r}_i^\top + \frac{1}{\|\boldsymbol{r}_i\|}(\boldsymbol{\dot p} - \boldsymbol{\dot p}_i)^\top(\boldsymbol{I} - \boldsymbol{\hat r}_i\boldsymbol{\hat r}^\top_i)\\ 
                \frac{\partial \psi_i}{\partial \boldsymbol{\dot p}} =& -\boldsymbol{\hat r}_i^\top\\ 
                \frac{\partial \psi_i}{\partial \boldsymbol{\dot p}_i} =& \boldsymbol{\hat r}_i^\top, \text{ and } \frac{\partial \psi_i}{\partial \boldsymbol{p}_j} = \frac{\partial \psi_i}{\partial \boldsymbol{\dot p}_j} = \boldsymbol{0}^\top, i\neq j.
            \end{aligned}
            \label{eq:hocbf-comp}
        \end{equation}
    }
    \begin{figure}
        \centering
        \includegraphics[width=\linewidth, trim={0 1cm 0 1cm},clip]{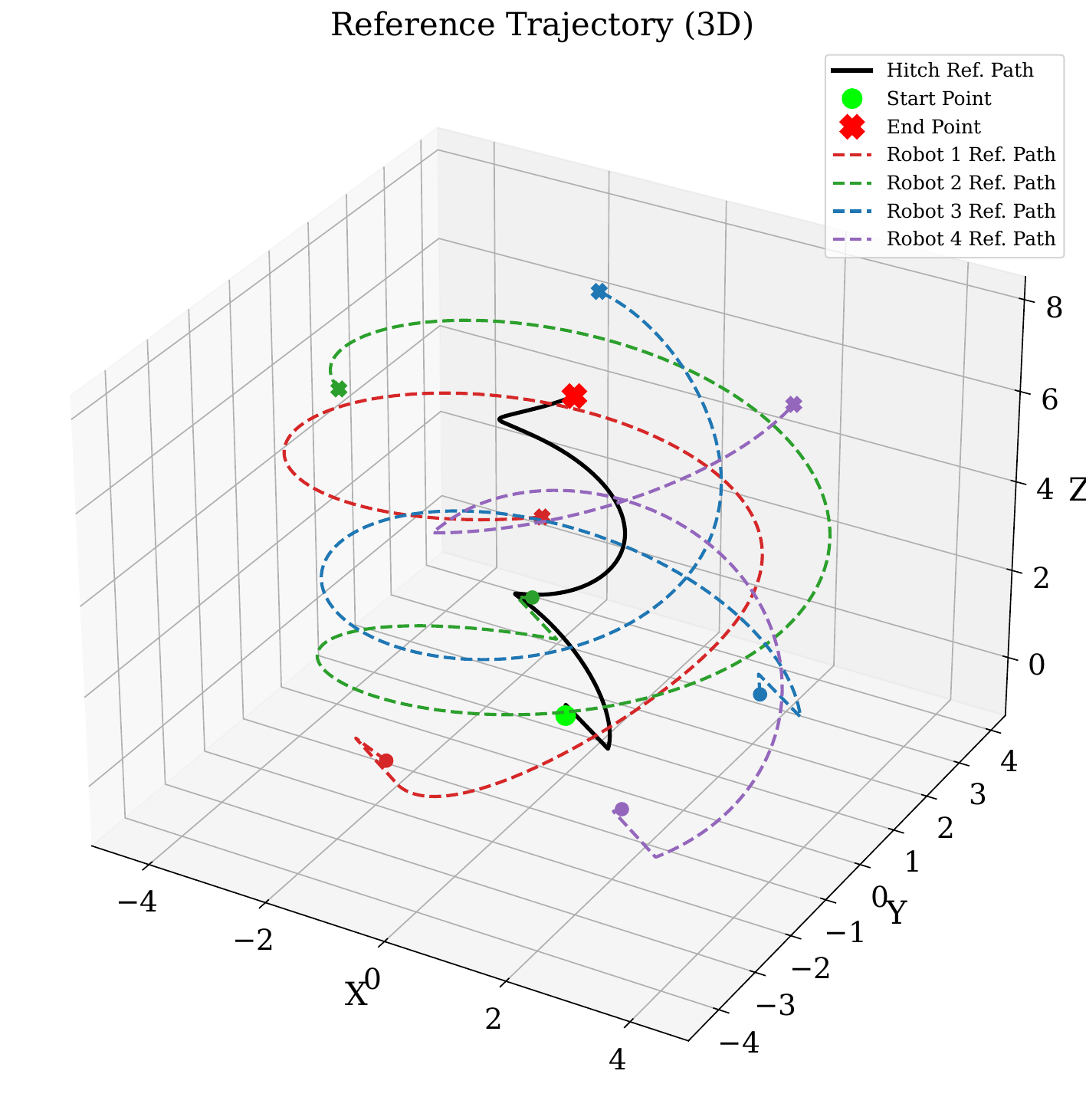}
        \caption{A example state trajectory in 3D in a $10$-s operation.}
        \label{fig:3d-traj}
    \end{figure}
    \subsection{Example Statistics in reference following}{
    Fig.~\ref{fig:3d-plots-fast} shows the evolution of Lyapunov function values and the sum of errors while the system is following a reference of which the hitch moves at $0.9$m/s for $30$s. 
    }

    \subsection{Example state trajectory in reference following}{
    Fig.~\ref{fig:3d-traj} shows the trajectory of the hitch and the robot positions while the system is following a reference of which the hitch moves at $0.9$m/s for $10$s. 
    }
}
\end{document}